\documentclass[letterpaper]{article} 
\usepackage{aaai25}  
\usepackage{times}  
\usepackage{helvet}  
\usepackage{courier}  
\usepackage[hyphens]{url}  
\usepackage{graphicx} 
\usepackage{natbib}  
\usepackage{caption} 
\usepackage{algorithm}
\usepackage{algorithmic}
\usepackage{amsmath}
\usepackage{amssymb}
\usepackage{amsthm}
\usepackage{subfigure}
\usepackage{tabularx}
\usepackage{makecell}

\theoremstyle{plain}
\newtheorem{theorem}{Theorem}

\newtheorem{lemma}[theorem]{Lemma}

\theoremstyle{definition}
\newtheorem{definition}[theorem]{Definition}

\theoremstyle{remark}

\usepackage[T1]{fontenc}
\usepackage{newfloat}
\usepackage{listings}
\DeclareCaptionStyle{ruled}{labelfont=normalfont,labelsep=colon,strut=off} 
\floatstyle{ruled}
\newfloat{listing}{tb}{lst}{}
\floatname{listing}{Listing}
\title{AIR: Unifying Individual and Collective Exploration \\in Cooperative Multi-Agent Reinforcement Learning}
\author {
    Guangchong Zhou\textsuperscript{\rm 1,\rm 2},
    Zeren Zhang\textsuperscript{\rm 1,\rm 2},
    Guoliang Fan\textsuperscript{\rm 1}
}
\affiliations {
    \textsuperscript{\rm 1}The Key Laboratory of Cognition and Decision Intelligence for Complex Systems,\\ Institute of Automation, Chinese Academy of Sciences\\
    \textsuperscript{\rm 2}School of Artificial Intelligence, University of Chinese Academy of Sciences\\
    \{zhouguangchong2021, zhangzeren2021, guoliang.fan\}@ia.ac.cn
}

\begin{document}

\maketitle

\begin{abstract}
Exploration in cooperative multi-agent reinforcement learning (MARL) remains challenging for value-based agents due to the absence of an explicit policy. Existing approaches include individual exploration based on uncertainty towards the system and collective exploration through behavioral diversity among agents. However, the introduction of additional structures often leads to reduced training efficiency and infeasible integration of these methods. In this paper, we propose Adaptive exploration via Identity Recognition~(AIR), which consists of two adversarial components: a classifier that recognizes agent identities from their trajectories, and an action selector that adaptively adjusts the mode and degree of exploration. We theoretically prove that AIR can facilitate both individual and collective exploration during training, and experiments also demonstrate the efficiency and effectiveness of AIR across various tasks.
\end{abstract}

%
\begin{links}
     \link{Code}{https://github.com/Jugg1er/AIR}
\end{links}

\section{Introduction}
Multi-agent reinforcement learning~(MARL) has achieved outstanding results in complex cooperative tasks, such as flocking control~\citep{xu2018multi,gu2023safe}, autonomous driving~\citep{shamsoshoara2019distributed,zhang2023spatial}, and sensor network~\citep{zhao2023optimizing}. To deal with the partial observability and the vast joint spaces of agents, most MARL methods follow the \textit{centralized training and decentralized execution} (CTDE) paradigm. Especially, value-based approaches~\cite{sunehag2017value,rashid2018qmix,yang2020qatten,wang2020qplex} enjoy high sample efficiency and great scalability, ultimately achieving superior performance in popular benchmarks. However, most value-based methods suffer from insufficient exploration due to the adoption of the vanilla $\epsilon$-greedy exploration strategy, and how to effectively enhance the exploration of value-based agents in cooperative tasks remains a challenging problem.

In some studies, the exploration of an agent is guided by its uncertainty towards the system. \citet{zhang2023self} define the uncertainty as the correlation between an agent's action and other agents' observations, based on which they adjust exploration rate $\epsilon$ dynamically. EITI \& EDTI~\cite{wang2019influence} and EMC~\cite{zheng2021episodic} refer to the concept of ``curiosity'', which is measured by the model prediction error in practice and used as an intrinsic reward to facilitate exploration. These approaches introduce perturbations into individual policies but do not explicitly consider coordination, and we call them individual exploration.

Compared to single-agent domain, multi-agent cooperative tasks typically require various skills and have vastly larger policy space, rendering comprehensive exploration infeasible for a single agent. An ideal solution is to keep the diversity of agents' behaviors, which enables individuals to acquire distinct skills, explore disparate regions of the policy space, and share their experiences by parameter sharing technique. Recently, a group of methods revealed the importance of behavioral diversity in collective exploration. As shown in Figure~\ref{fig:sim}, all agents learn to compete for the ball and form terrible cooperation, while the globally optimal policy in Figure~\ref{fig:div} induces different behaviors among agents for tacit cooperation~\cite{li2021celebrating}. Such sub-optimal situations arise exclusively in multi-agent tasks and can not be alleviated through individual exploration. MAVEN~\cite{mahajan2019maven} proposes committed exploration that learns a diverse ensemble of monotonic approximations with the help of a latent space to explore. CDS~\cite{li2021celebrating} maximizes an information-theoretical regularizer to produce diverse individualized behaviors for extensive exploration. In role-based methods~\cite{wang2020roma,wang2020rode,yang2022ldsa,zhou2023sora}, agents are required to deal with distinct subtasks based on their assigned roles, thereby reducing conflicts and accelerating the formation of cooperation.

\begin{figure}[t]
    \centering
    \subfigure[Sub-optimal policy.]{\includegraphics[width=0.44\linewidth]{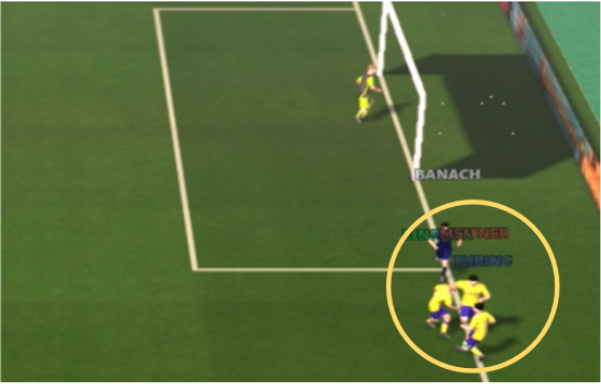}\label{fig:sim}}
    \subfigure[Optimal policy.]{\includegraphics[width=0.44\linewidth]{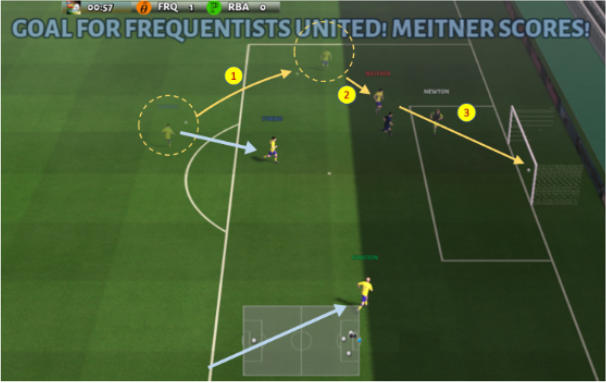}\label{fig:div}}
    \caption{The importance of behavioral diversity in Google Research Football. (a) Agents all compete for the ball, exhibiting homogeneous behaviors and poor coordination. (b) Agents behave differently to achieve coordination.}
    \label{fig:football}
\end{figure}

Both individual skills and group collaboration are crucial in cooperative tasks. However, even without methodological conflicts, a direct integration of the above individual and collective exploration methods is infeasible, as excessive extra modules make the overall model too heavy to train. To the best of our knowledge, no research has integrated individual and collective exploration into a unified framework. To address this gap, we propose Adaptive exploration via Individual Recognition~(AIR). We first learn an identity classifier that distinguishes corresponding agents based on given trajectories. Based on the classifier, we conduct theoretical analysis and design a unified exploration mechanism that can switch the exploration mode between individual and collective exploration according to the training stage as well as tuning the degree of exploration. Our main contributions can be summarized as follows:
\begin{itemize}
    \item We analyze the impact and limitations of traditional exploration strategies within multi-agent systems and reveal that multi-agent exploration accounts for both individual behaviors and the joint behaviors of agents. 
    \item We propose a novel multi-agent exploration framework for value-based agents called AIR, which enables dynamic adjustments of exploration during training. To the best of our knowledge, we are the very first to address individual and collective exploration in a single unified framework.
    \item Our proposed method incorporates only the minimum additional structure and is computationally efficient. Experiments demonstrate the remarkable effectiveness of AIR across various multi-agent tasks.
\end{itemize}

\section{Related Works}
\subsection{Value-based MARL Methods}
Independent Q-learning~\cite{tampuu2017multiagent} directly migrates DQN into multi-agent tasks, but it completely ignores the interactions between agents, leading to various training issues. Value decomposition factorizes the global value into individual ones and alleviates the instability during training. VDN~\cite{sunehag2017value} and QMIX~\cite{rashid2018qmix} aggregate individual Q-values into a global one through additive and monotonic functions respectively. Qatten~\cite{yang2020qatten} replaces the dense network in QMIX with the attention mechanism. To further enhance the capability of the mixing network to express richer joint action-value function classes, QTRAN~\cite{DBLP:journals/corr/abs-1905-05408} constructs two soft regularizations to align the greedy action selections between the joint and individual value functions. Weighted-QMIX~\cite{rashid2020weighted} introduces a weighting mechanism into the projection of monotonic value factorization to place more importance on better joint actions. QPLEX~\cite{wang2020qplex} proposes a duplex dueling network architecture to loose the monotonicity constraints in QMIX. QPro~\cite{mei2024projection} casts the factorization problem as regret minimization over the projection weights of different state-action values.

\subsection{Exploration for Value-based Agents in MARL}
Despite their success, the simple $\epsilon$-greedy exploration strategy used in the above methods has been found ineffective in solving coordination tasks with complex state and reward transitions. Some studies adjust the intensity of individual exploration dynamically according to the metric of uncertainty towards the system. SMMAE~\cite{zhang2023self} refers to the correlation between each agent's action and others' observations. EITI \& EDTI~\cite{wang2019influence} quantify and characterize the influence of one agent’s behavior on others using mutual information and the difference of expected returns respectively, while EMC~\cite{zheng2021episodic} uses prediction errors of individual Q-values to capture the novelty of states and the influence from other agents. 

Other studies have realized the importance of behavioral diversity in exploration in multi-agent tasks. MAVEN~\cite{mahajan2019maven} uses a hierarchical policy to produce a shared latent variable and learns several state-action value functions for each agent to explore different trajectories. CDS~\cite{li2021celebrating} proposes a novel information-theoretical objective to maximize the mutual information between agents’ identities and trajectories to encourage diverse individualized behaviors. Role-based methods~\cite{wang2020roma,wang2020rode,yang2022ldsa,zhou2023sora} decompose the complex task to subtasks, in which each agent search in a sub-region of the entire observation-action space, thereby reducing training complexity. The drawback of these methods lies in their introduction of additional modules to the policy network, which increases the complexity of the decision-making process and consequently reduces both training and execution efficiency.

\section{Preliminaries}
\subsection{Dec-POMDP}
A fully cooperative multi-agent system (MAS) is typically represented by a decentralized partially observable Markov decision process (Dec-POMDP) \cite{oliehoek2016concise}, which is composed of a tuple $G = \langle\mathcal{S},\boldsymbol{\mathcal{U}}, \mathcal{P}, \mathcal{Z}, r, \mathcal{O}, n, \gamma\rangle$. At each time-step, the current global state of the environment is denoted by $s \in \mathcal{S}$, while each agent $a \in \mathcal{A} := \{1, \ldots, n\}$ only receives a unique local observation $z_a \in \mathcal{Z}$ generated by the observation function $\mathcal{O}(s, a): \mathcal{S} \times \mathcal{A} \rightarrow \mathcal{Z}$. Subsequently, every agent $a$ selects an action $u_a \in \mathcal{U}$, and all individual actions are combined to form the joint action $\boldsymbol{u}=[u_1, \ldots, u_n] \in \boldsymbol{\mathcal{U}} \equiv \mathcal{U}^n$. The interaction between the joint action $\boldsymbol{u}$ and the current state $s$ leads to a change in the environment to state $s^{\prime}$ as dictated by the state transition function $\mathcal{P}(s'|s, \boldsymbol{u}): \mathcal{S} \times \mathcal{U} \times \mathcal{S} \rightarrow [0, 1]$. All agents in the Dec-POMDP share the same global reward function $r(s, \boldsymbol{u}): \mathcal{S} \times \boldsymbol{\mathcal{U}} \rightarrow \mathbb{R}$, and $\gamma \in [0, 1)$ represents the discount factor.

\subsection{Value Decomposition}
Credit assignment is a key problem in cooperative MARL which helps each agent to discern its contribution to the collective performance. To solve this, value decomposition~(VD) methods assume that each agent has a specific value function, and the integration of all individual functions creates the joint value function. To guarantee that the optimal action for each agent aligns with the global optimal joint action, most value decomposition methods satisfy the \textit{Individual Global Max} (IGM) assumption:
\begin{equation*}
    \arg \max _{\boldsymbol{u}} Q_{tot}(\boldsymbol{\tau}, \boldsymbol{u})=\arg \max _{u_a} Q_a(\tau_a, u_a),\quad \forall a\in \mathcal{A}
\end{equation*}
where $Q_{tot}=f(Q_1, ..., Q_n)$, $Q_1, ..., Q_n$ denote the individual Q-values, and $f$ is the mixing function. Various value decomposition methods, including VDN~\cite{sunehag2017value}, QMIX~\cite{rashid2018qmix} and so on, arise from variations in the construction of $f$. The model is commonly learnt by the temporal difference (TD) error below, in which $Q_{tot}^-$ is a periodically updated target network of $Q_{tot}$.
\begin{equation*}
    \mathcal{L}_{\textrm{TD}}=\left[r_t + \gamma \max_{\boldsymbol{u}_{t+1}}Q_{tot}^-(\boldsymbol{\tau}_{t+1}, \boldsymbol{u}_{t+1})-Q_{tot}(\boldsymbol{\tau}_{t}, \boldsymbol{u}_{t})\right]^2
\end{equation*}

\section{Method}
Compared to single-agent reinforcement learning, exploration in MARL also necessitates behavioral diversity across agents. In this section, we introduce a novel strategy called Adaptive exploration via Individual Recognition~(AIR), which integrates individual and collective exploration into a unified framework. We first formulate a discrepancy metric of trajectories to measure the behavioral diversity among different agents and obtain the learning objective for collective exploration by theoretical derivation. With this objective, we further design an adversarial strategy to stimulate each agent to deviate from a fixed policy and explore novel actions. Finally, we integrate the collective and individual exploration into a unified framework.

\subsection{Diverse Behaviors across Agents}
We begin the derivation of AIR by defining the concept of \textit{trajectory visit distribution}:
\begin{definition}[Trajectory visit distribution]
    In a multi-agent system, the distribution of the trajectory experienced by the agent $k$ is called individual trajectory visit distribution and is defined as:
    \begin{align}
        \rho^k(\tau_T, u_T)&=\rho(\tau_T, u_T|\textrm{agent\_id}=k) \nonumber\\
        &=\rho(o_0, u_0, o_1, u_1, ..., o_T, u_T | z_k) \nonumber\\
        &=\prod\limits_{t=0}^T \pi^k(u_t | o_t) \sum\limits_{s_t} P(s_t) O(o_t|s_t, k),
        \label{eq:rho_k}
    \end{align}
    where $P(s_t)$ is the probability of the state being $s_t$ at time step $t$, $O(o_t|s_t, k)$ is the probability of agent $k$ receiving observation $o_t$ under state $s_t$, and $\pi^k$ is the policy of agent $k$. The trajectory visit distribution of the overall system containing $n$ agents is called system trajectory visit distribution and is defined as:
    \begin{equation}
        \rho = \frac{1}{n} \sum\limits_{k=1}^n \rho^k .
        \label{eq:rho}
    \end{equation}
\end{definition}

The trajectory visit distribution, which indicates the probability of encountering a specific trajectory with a given policy from initial state $s_0$, encapsulates the diversity of policies naturally. Therefore, the diverse behaviors of agents can be assessed by measuring the diversity of trajectories generated by distinct individual policies. The Kullback-Leibler~(KL) divergence is a commonly used metric for comparing the discrepancy between distributions, by which we define the difference between an agent's policy and other agents' policies as below:
\begin{definition}[Difference between individual policies]
    In a multi-agent system with $N$ agents, the difference between the policy of agent $k$ and the policies of other agents can be defined by the discrepancy of their generated trajectories:
    \begin{equation}
        \mathcal{D}_{\textrm{KL}}\left[\rho^k || \rho\right]=\mathcal{D}_{\textrm{KL}}\left[\rho(\tau, u|z_k) || \rho(\tau, u)\right]
        \label{eq:dbp}
    \end{equation}
\end{definition}

According to Equation~\ref{eq:dbp}, we improve the difference between an individual trajectory visit distribution and the system trajectory visit distribution to enforce diverse behaviors of agents. However, it is completely infeasible to estimate $\rho(\tau, u)$ directly due to an intractably huge input space. We notice that Equation~\ref{eq:dbp} can be expressed by the form of entropy discrepancy during training:
\begin{equation}
    \mathbb{E}_{z}\left[\mathcal{D}_{\textrm{KL}}\left(\rho(\tau, u|z) || \rho(\tau, u)\right)\right] = \mathcal{H}(\rho)-\mathcal{H}(\rho | z),
    \label{eq:ent}
\end{equation}

\noindent in which $\mathcal{H}(\rho)$ is the Shannon Entropy of trajectory visit distribution. The proof of Equation~\ref{eq:ent} can be found in Appendix~\ref{sec:lemma1}. With the mutual information theory, it can be further transformed into:
\begin{equation}
    \mathcal{H}(\rho)-\mathcal{H}(\rho | z) = \mathcal{H}(z)-\mathcal{H}(z | \rho)
\end{equation}

This optimization target is easier to solve in practice. As each agent $k$ has the same number of sampled trajectories in a training step, $\rho(z=k)$ follows a uniform distribution and thus $\mathcal{H}(z)=-\sum\limits_{k=1}^n \frac{1}{n} \log \frac{1}{n}=\log n$ is a constant. The second term $\mathcal{H}(z | \rho)$ implies that the identity of an agent can be inferred given its trajectory and selected action. We approximate true posterior $\rho(z|\tau, u)$ with a learned discriminator $q_{\zeta}(z| \tau, u)$ and optimize it using the evidence lower bound~(ELBO).
\begin{align}
    \mathcal{H}(z)-\mathcal{H}(z | \rho) &=\mathbb{E}_{\tau, u, z}\left[\log \rho(z| \tau, u)\right] + \log n\nonumber \\
    &= \mathbb{E}_{\tau, u}\left[\mathcal{D}_{\textrm{KL}}\left(\rho(z|\tau, u)||q_{\zeta}(z|\tau,u)\right)\right] \nonumber \\
    &+ \mathbb{E}_{\tau, u, z}\left[\log q_{\zeta}(z|\tau,u)\right] + \log n \nonumber \\
    & \geq \mathbb{E}_{\tau, u, z}\left[\log q_{\zeta}(z|\tau,u)\right] + \log n,
\end{align}

\noindent where we exploit the non-negativity of KL divergence. Therefore, with the sub-goal of promoting behavioral diversity, the policy of each agent $k$ can be expressed as below:
\begin{equation}
    u^* = \arg \max_{u} \left[Q^k (\tau, u) + \alpha \log q_{\zeta}(z_k|\tau, u)\right],
    \label{eq:ce}
\end{equation}

\noindent where $\alpha$ is a positive temperature controlling the degree of collective exploration. For more precise approximation, we further tighten the lower bound by minimizing $\mathcal{D}_{\textrm{KL}}\left(\rho(z|\tau, u)||q_{\zeta}(z|\tau,u)\right)$ and derive the gradient of $q_{\zeta}$ as shown in Equation~\ref{eq:J_zeta}.
\begin{align}
    &\nabla_{\zeta} \mathbb{E}_{\tau, u}\left[\mathcal{D}_{\textrm{KL}}\left(\rho(z|\tau, u)||q_{\zeta}(z|\tau, u)\right)\right] \nonumber \\
    =& -\mathbb{E}_{\tau,u} \left[\nabla_{\zeta}\log q_{\zeta}(z| \tau, u)\right].
    \label{eq:J_zeta}
\end{align}

\subsection{Individual Exploration of Value-based Agents}
In essence, exploration at the individual level drives deviation from established policies, prompting agents to take unfamiliar and unknown actions. A common individual exploration strategy used in policy gradient methods is to add the entropy of policy $\pi_\theta$ as a regularization term in the optimization target $J(\theta)$. While it can adjust the degree of exploration dynamically during training, this approach does not work for VD methods because the value-based agent makes decision according to its value estimations towards the given observation and possible actions instead of an explicit policy distribution.

To facilitate exploration for VD methods, we introduce a novel adversarial mechanism somewhat similar to generative adversarial networks~(GAN)~\cite{goodfellow2014generative}. Under the CTDE framework, we introduce two adversarial components: a centralized identity classifier trained to identify the agent responsible for generating the given trajectory-action pair (we directly reuse the discriminator $q_{\zeta}$ for collective exploration in practice), and a decentralized action selector that attempts to mislead the classifier. This adversarial mechanism can be interpreted by the following relationship (the proof is in Appendix~\ref{sec:lemma3}):
\begin{equation}
    p(z_k|\tau_T, u_T)=\frac{\prod\limits_{t=0}^T \pi^k(u_t|o_t)}{\prod\limits_{t=0}^T \pi^k(u_t|o_t) + \sum\limits_{\substack{i=1\\ i\neq k}}^n \prod\limits_{t=0}^T \pi^i(u_t|o_t)}
    \label{eq:relation}
\end{equation}

According to Equation~\ref{eq:relation}, the higher the probability that an action $u_t$ in the trajectory is selected by the policy, the higher the likelihood that the classifier will correctly identify the source agent $k$ of the trajectory, and vice versa. Thus, even without an explicit policy function, we can leverage the classifier to measure action selection probabilities and encourage the agent to explore low-probability actions, achieving a similar effect to the entropy-based approach. Without distorting the individual Q-value estimation, the action selector promotes the agent $k$ to choose unexplored actions with low posterior probabilities $q_{\zeta}(z_k|\tau, u)$. We formally describe how the action selector works as:
\begin{align}
    u&=\arg \max_u\tilde{Q}^k(\tau, u)\nonumber \\
    &=\arg \max_u \left[ Q^k(\tau, u) - \alpha \log q_{\zeta}(z_k|\tau, u)\right]
\label{eq:ie}
\end{align}

Equation~\ref{eq:ie} has a very similar form to Equation~\ref{eq:ce}, and we take Equation~\ref{eq:ie} as the standard form. The value estimations of less exploited actions are raised more to increase their chances of being chosen when $\alpha >0$, thereby enhancing the individual exploration of the agent. When $\alpha <0$, it turns to collective exploration to induce diverse and cooperative behaviors among agents.

\subsection{Adaptive Temperature}
In Equation~\ref{eq:ie}, the tuning of temperature value $\alpha$ is closely related to the magnitude of Q-value estimation, which is further determined by the magnitude of received rewards that differs not only across tasks but also over the training process due to the changing policy. Therefore, choosing the optimal temperature $\alpha^*$ is non-trivial. Besides, forcing the temperature to a fixed value during training is a poor solution, since the mode and degree of exploration are expected to be adjusted adaptively at different training stages. To dynamically tune the temperature during training, we formulate the learning target as a constrained optimization problem for the agent $k$:
\begin{align}
    &\max_{\pi_{0:T}^k} \mathbb{E}_{\rho} \left[\sum\limits_{t=0}^T r(\boldsymbol{\tau}_t, \boldsymbol{u}_t)\right] \equiv \max_{\pi_{0:T}^k} \mathbb{E}_{\rho^k} \left[\sum\limits_{t=0}^T r(\tau_t, u_t)\right] \nonumber \\
    &\textrm{s.t.} \; \mathbb{E}_{(\tau_t, u_t) \sim \rho^k} \left[-\log q_{\zeta}(z_k|\tau_t, u_t)\right] \geq \mathcal{H}\quad \forall t,
    \label{eq:constraint}
\end{align}

\noindent where $r(\boldsymbol{\tau}_t, \boldsymbol{u}_t)$ is the global reward function and $r(\tau_t, u_t)$ is the factored individual reward at time step $t$ under IGM assumption, and $\mathcal{H}$ is a minimum expected entropy. Since the policy and trajectory at $t$ would affect the future calculation, we rewrite the objective in a recursive form to perform dynamic programming backward:
\begin{equation}
    \max_{\pi_0^k}\left(\mathbb{E}\left[r(\tau_0, u_0)\right] + \max_{\pi_1^k}\left(\mathbb{E}\left[...\right] + \mathbb{E}_{\pi_T^k}\left[r(\tau_T, u_T)\right]\right)\right),
    \label{eq:max}
\end{equation}

\noindent subjected to the same constraint in Equation~\ref{eq:constraint}. To focus on the temperature, we turn to the dual problem instead. We start with the last time step $T$. With the constraint in Equation~\ref{eq:constraint},
\begin{align}
    \label{eq:dual}
    &\max_{\pi_T^k} \mathbb{E}_{(\tau_T, u_T)\sim \rho^k} \left[r(\tau_T, u_T)\right] \\
    = &\min_{\alpha_T \geq 0} \max_{\pi_T^k}\mathbb{E}\left[r(\tau_T, u_T)-\alpha_T \log q_{\zeta}(z_k|\tau_T, u_T)\right]-\alpha_T \mathcal{H}, \nonumber
\end{align}

\noindent where $\alpha_T$ is the dual variable. As $\pi_T^k$ is in the form of Equation~\ref{eq:ie} and only related to $q_{\zeta}$ when the Q-value function is fixed, we can replace $\pi_T^k$ in the subscripts of Equation~\ref{eq:dual} with $q_{\zeta}$. Here we employ strong duality since the objective is linear and the constraint in Equation~\ref{eq:constraint} is a convex function with respect to $q_{\zeta}$. With the correspondence between the optimal policy deduced by $q_{\zeta}^*$ and the temperature value $\alpha_T$, the optimal dual variable $\alpha_T^*$ is solved by
\begin{equation}
    \arg \max_{\alpha_T} \mathbb{E}_{(\tau_t, u_t)\sim q_{\zeta}^*} \left[\alpha_T \log q_{\zeta}(z_k|\tau_T, u_T)+\alpha_T \mathcal{H}\right].
\end{equation}

To simplify notation, we refer to the soft Bellman equation~\cite{haarnoja2017reinforcement}
\begin{align}
    &Q^*(\tau_t, u_t) \\
    =&r(\tau_t, u_t) + \mathbb{E}_{\rho^k}\left[Q^*(\tau_{t+1}, u_{t+1})-\alpha \log q_{\zeta}(z_k|\tau_{t+1}, u_{t+1})\right] \nonumber
\end{align}

\noindent with $Q_T^*(\tau_T, u_T)=\mathbb{E}[r(\tau_T, u_T)]$. We now take Equation~\ref{eq:max} a step further using the dual problem and get:
\begin{align}
    & \max_{\pi_{T-1}^k}\left(\mathbb{E}[r(\tau_{T-1}, u_{T-1})]+\max_{\pi_T^k}\mathbb{E}[r(\tau_T, u_T)]\right) \nonumber \\
    =& \max_{\pi_{T-1}^k} \left(Q^*(\tau_{T-1}, u_{T-1})-\alpha_T \mathcal{H}\right) \\
    =& \min_{\alpha_{T-1}\geq 0} \max_{T-1}\big (\mathbb{E}\left[Q^*(\tau_{T-1}, u_{T-1})\right] \nonumber \\
    -& \mathbb{E}\left[\alpha_{T-1} \log q_{\zeta}(z_k|\tau_{T-1}, u_{T-1})\right]-\alpha_{T-1}\mathcal{H}\big ) + \alpha_T^* \mathcal{H}. \nonumber
\end{align}

Thus Equation~\ref{eq:constraint} can be optimized recursively from back to front. After obtaining the optimal Q-function $Q^*$ and the corresponding optimal policy $\pi_t^*$, the optimal dual variable $\alpha_t^*$ can be solved by:
\begin{equation}
    \alpha_t^* = \arg \max_{\alpha_t} \mathbb{E}_{u_t\sim \pi_t^*} \left[\alpha_t \log q_{\zeta}(z_k|\tau_t, u_t)+\alpha_t \mathcal{\bar{H}}\right].
    \label{eq:optimal_temp}
\end{equation}

Equation~\ref{eq:optimal_temp} provides a theoretical approach for calculating the optimal temperature, but it is not feasible in practice since we can not obtain $Q^*$ directly but iteratively update the Q-function to approximate it. Therefore, we refer to dual gradient descent~\cite{boyd2004convex}, which alternates between finding the optimal values of primal variables with fixed dual variables and applying gradient descent on dual variables for one step. It is also noted that while optimizing with respect to the primal variables fully is impractical, a truncated version that performs incomplete optimization (even for a single gradient step) can be shown to converge under convexity assumptions. Although the assumptions are not strictly satisfied, \citet{haarnoja2018soft} still extended the technique to neural networks and found it working in practice. Thus, we compute gradients for $\alpha$ according to the objective in Equation~\ref{eq:J_alpha}. $\mathcal{\bar{H}}$ is a hyperparameter and is calculated as the running mean of $-\log q_{\zeta}(z_k|\tau_t, u_t)$ in practice. The overall algorithm is displayed in Algorithm~\ref{alg:AIR}.
\begin{equation}
    J(\alpha)=\mathbb{E}_{u_t \sim \pi_t^k} \left[\alpha \log q_{\zeta}(z_k|\tau_t, u_t) + \alpha \mathcal{\bar{H}}\right] 
    \label{eq:J_alpha}
\end{equation}

\begin{algorithm}[tb]
\caption{Adversarial Identity Recognition~(AIR)}
\label{alg:AIR}
\textbf{Parameters}: Value decomposition framework $\theta$, target network $\theta^-=\theta$, identity classifier $\zeta$, initial temperature $\alpha_0$.
\begin{algorithmic}[1] 
\STATE Initialize an empty replay buffer $\mathcal{D}$.
\WHILE{\textit{training}}
\FOR {$episode \leftarrow 1$ \TO $M$}
\STATE Initialize $\boldsymbol{E}=\emptyset$.
\FOR {each time step $t$}
\FOR {each agent $k$}
\STATE Select action based on $\tilde{Q}^k$ in Equation~\ref{eq:ie}.
\STATE Store the transition in $\boldsymbol{E}$.
\ENDFOR
\ENDFOR
\STATE Add the episodic data $\boldsymbol{E}$ to $\mathcal{D}$.
\ENDFOR
\STATE Sample a mini-batch $\mathcal{B}$ from $\mathcal{D}$.
\STATE Follow the value decomposition method to update $\theta$,
\STATE $\zeta \leftarrow \zeta + \lambda_{\zeta} \hat{\nabla_{\zeta}}J(\zeta)$,
\STATE $\alpha \leftarrow \alpha + \lambda_{\alpha} \hat{\nabla_{\alpha}}J(\alpha)$,
\STATE Update target network $\theta^- \leftarrow \theta$ periodically.
\ENDWHILE
\end{algorithmic}
\end{algorithm}

\section{Experiments}
In this section, we conduct a large set of experiments to validate the efficiency and effectiveness of AIR across various scenarios. To be consistent with previous methods~\cite{mahajan2019maven,wang2020rode,yang2022ldsa}, we similarly employ QMIX~\cite{rashid2018qmix} as the backbone while using AIR for exploration. Several state-of-the-art MARL methods are selected as baselines, including QMIX~\cite{rashid2018qmix}, QPLEX~\cite{wang2020qplex}, and recognized works about multi-agent exploration (MAVEN~\cite{mahajan2019maven}, RODE~\cite{wang2020rode}, LDSA~\cite{yang2022ldsa}). All the baselines are implemented with the open-source codes in their original papers. We run each algorithm on 5 different random seeds for every task to alleviate the randomness in experiments.

\subsection{StarCraft II Multi-Agent Challenges}
StarCraft II Multi-Agent Challenges~(SMAC)~\citep{samvelyan2019starcraft} is a popular benchmark for cooperative MARL. There are two armies of units in SMAC. Each ally unit is controlled by a decentralized agent that can only act based on its local observation and the enemy units are controlled by built-in handcrafted heuristic rules. The performance of an MARL method is assessed by its test win rates in the SMAC scenarios. In our experiments, the difficulty level of built-in AI is set to 7~(very hard), and the version of StarCraft II engine is 4.6.2 instead of the simpler 4.10. 

We present experiment results on 6 \textit{Hard} and \textit{Super Hard} SMAC scenarios in Figure~\ref{fig:smac_results}. The solid line represents the median win rates during training, with the 25-75\% percentiles being shaded. Although the win rate in \textit{corridor} is slightly lower than RODE, AIR reaches the best performances in 5 out of all 6 scenarios with both the highest win rates and the earliest rise-ups. It is noteworthy that different scenarios necessitate distinct algorithmic capabilities. For example, agents in \textit{corridor} should adopt specialized roles alternately, with some drawing enemies and others focusing fire on enemies, which requires the algorithm to induce diverse agent behaviors. \textit{3s5z\_vs\_3s6z} and \textit{6h\_vs\_8z} are two hardest scenarios that require extensive exploration~\citep{hu2021rethinking}. While other multi-agent exploration methods (MAVEN, RODE and LDSA) only work on \textit{corridor} and completely fail in \textit{3s5z\_vs\_3s6z} and \textit{6h\_vs\_8z}, we are delighted to find that our proposed method exhibits remarkable sample efficiency and effectiveness across scenarios, which strongly demonstrates the superiority of AIR's multi-agent exploration strategy. 

\begin{figure*}[t]
    \centering
    \includegraphics[width=0.95\textwidth]{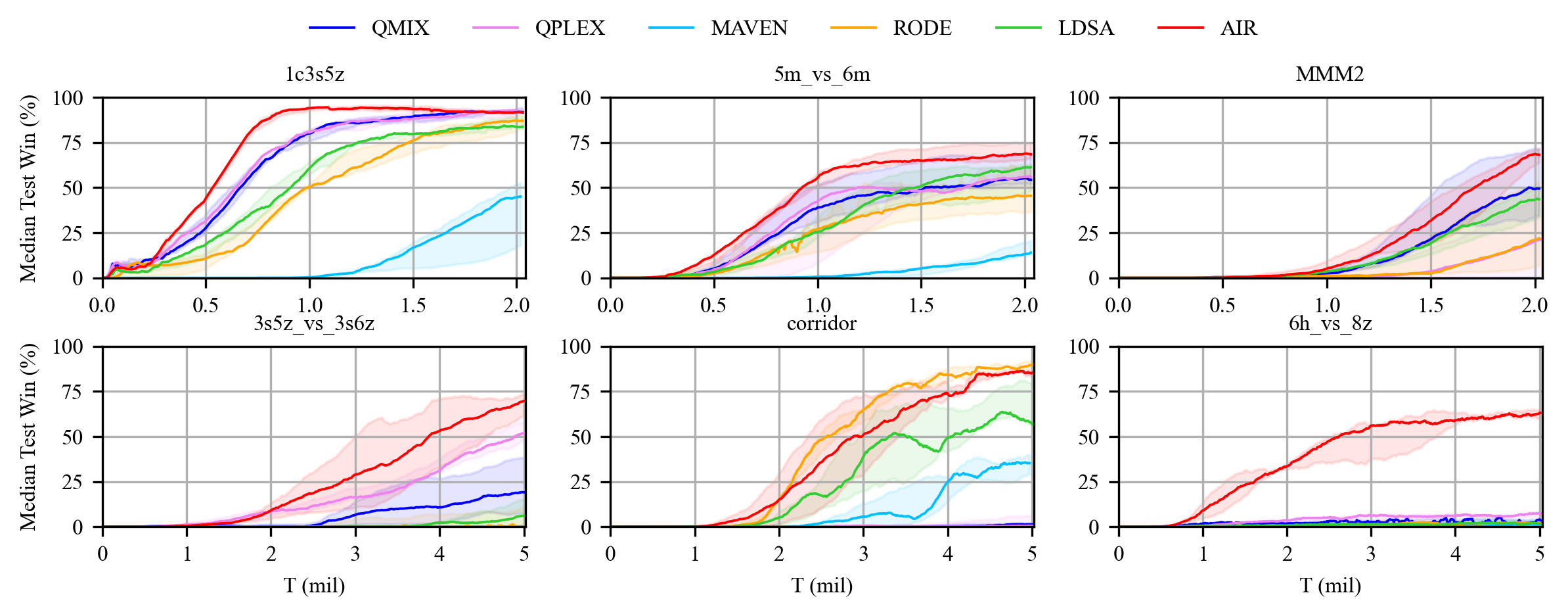}
    \caption{Experiment results of AIR and baselines on SMAC.}
    \label{fig:smac_results}
\end{figure*}

\subsection{Google Research Football}

\begin{figure*}[t]
    \centering
    \includegraphics[width=0.95\textwidth]{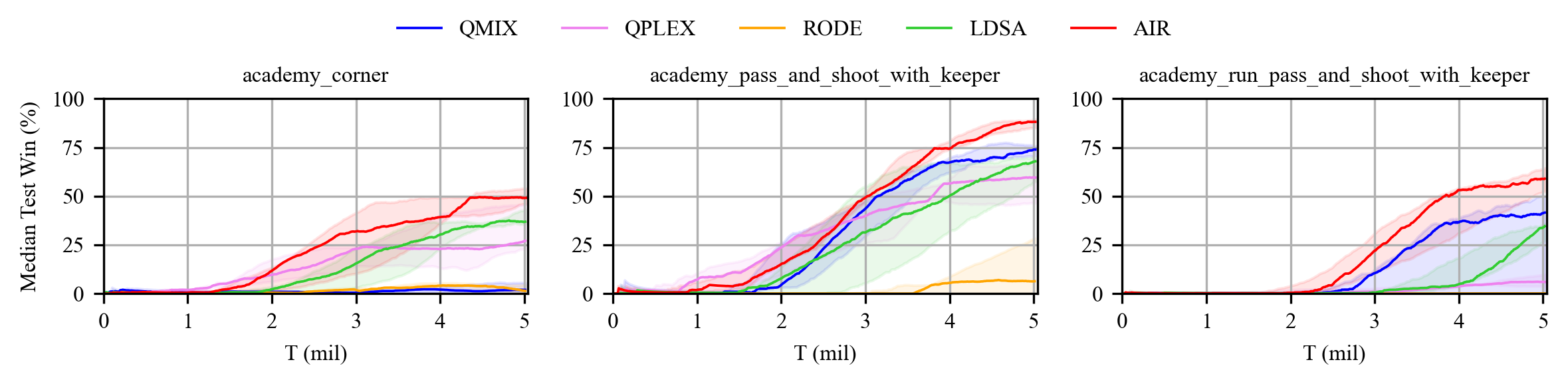}
    \caption{Experiment results of AIR and baselines on GRF.}
    \label{fig:grf_results}
\end{figure*}

Experiments in numerous previous studies have revealed that algorithms claiming state-of-the-art performance often exhibit significant degradation when applied to different environments. To assess the generalizability of AIR to other environments, we conducted further experiments in Google Research Football~(GRF)~\cite{kurach2020google}. GRF is a user-friendly platform for researchers to explore and develop novel RL algorithms with customized scenarios, which necessitates the acquisition of sophisticated strategies involving teamwork, ball control, and tactical decision-making. 

We choose three official scenarios from Football Academy, and the experiment results are displayed in Figure~\ref{fig:grf_results}. \textit{Academy\_pass\_and\_shoot\_with\_keeper} is a relatively easy scenario, in which AIR performs closely with the best baselines. \textit{Academy\_corner} and \textit{academy\_run\_pass\_and\_shoot\_with\_keeper} are two harder scenarios that requires tacit cooperation of agents. We can see that most of the baselines can address only one of the two scenarios, whereas AIR achieves significantly superior performance in both. Although RODE is considered a state-of-the-art method and even outperforms AIR in SMAC \textit{corridor}, it experiences a substantial performance decline in GRF. In contrast, AIR maintains its superiority in both environments. Therefore, we can conclude that AIR not only exhibits exceptional performance across diverse scenarios within the same environment but also maintains its superiority when applied to different environments.

\subsection{Case Study}
\begin{figure}
    \centering
    \includegraphics[width=1\linewidth]{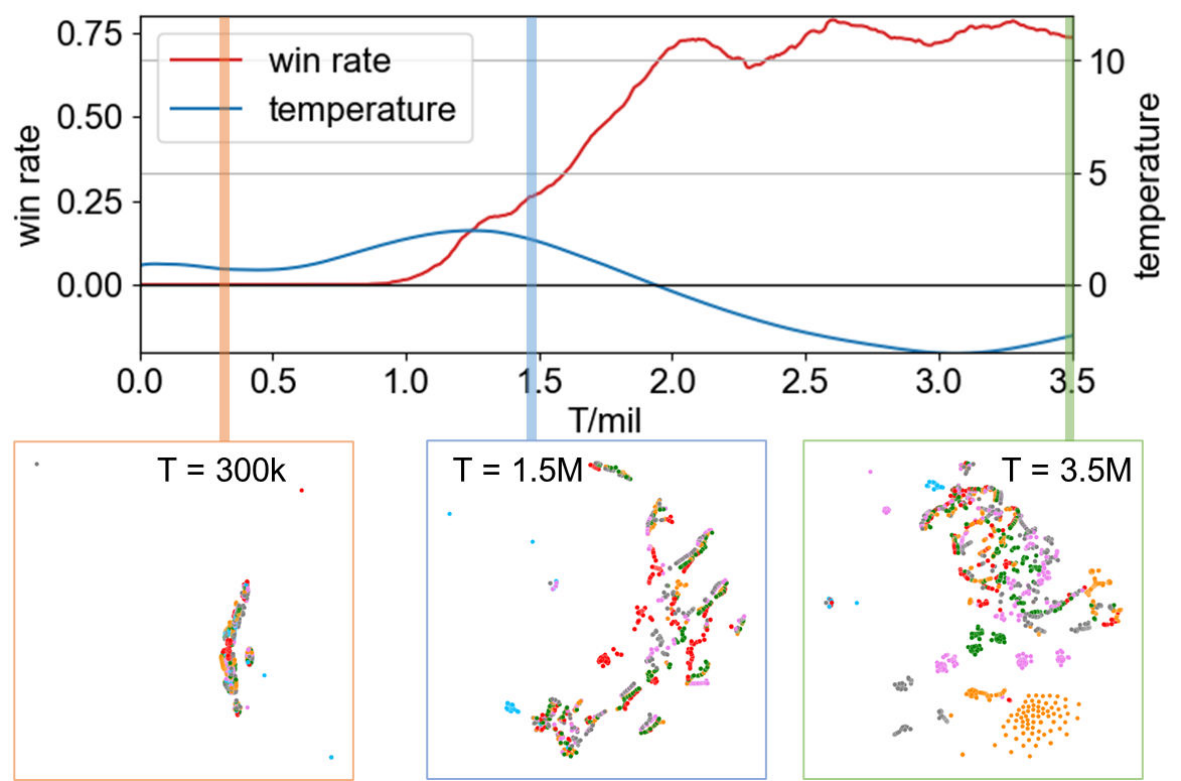}
    \caption{The training session on SMAC \textit{corridor}. \textbf{Up}: The curves that depict the changes of win rate and temperature value during training. \textbf{Down}: The 2D t-SNE visualizations of the agents' trajectories at different training steps.}
    \label{fig:case-study}
\end{figure}

In order to study how AIR adjusts the exploration during training, we dive into a training session conducted on the \textit{corridor} scenario in SMAC. We record the win rate and temperature values during training, and we reduce the high-dimensional observation vectors to 2D space via the t-SNE~\cite{van2008visualizing} algorithm for dimensionality reduction to visualize the distributions of agents' trajectories at different training steps. 

According to Figure~\ref{fig:case-study}, the temperature $\alpha$ is positive during the first half of the training process, thus the individual exploration is enhanced. At $T=300k$, the region visited by the agents is very limited, with few noises distributed away from the cluster. The degree of individual exploration reaches the peak at around $T=1.5M$, when all the agents visit a much wider region than they did at $T=300k$ steps, which confirms an extensive individual exploration between the two timestamps. After 2 million training steps, the temperature turns negative and the exploration mode is switched to collective exploration. At $T=3.5M$, the trajectories of different agents exhibit initial proximity at the beginning of an episode, but subsequently diverge and become readily distinguishable, while at $T=1.5M$ the trajectories of agents are largely intertwined. This suggests an obvious increase in the diversity of agents' behaviors. Therefore, we can conclude that AIR integrates individual and collective exploration and can dynamically adjust the mode and degree of exploration during training.

\subsection{Ablation Study}
To ensure a fair comparison in performance, it is necessary to exclude the effect of model size. We set the model size of QMIX as the baseline ($100\%$) and display the relative model sizes of other algorithms with respect to QMIX in Figure~\ref{fig:air_size}. Despite utilizing an extremely lightweight model, AIR still achieves significant results across tasks, demonstrating its great efficiency and superiority. 

\begin{figure}
    \centering
    \includegraphics[width=0.95\linewidth]{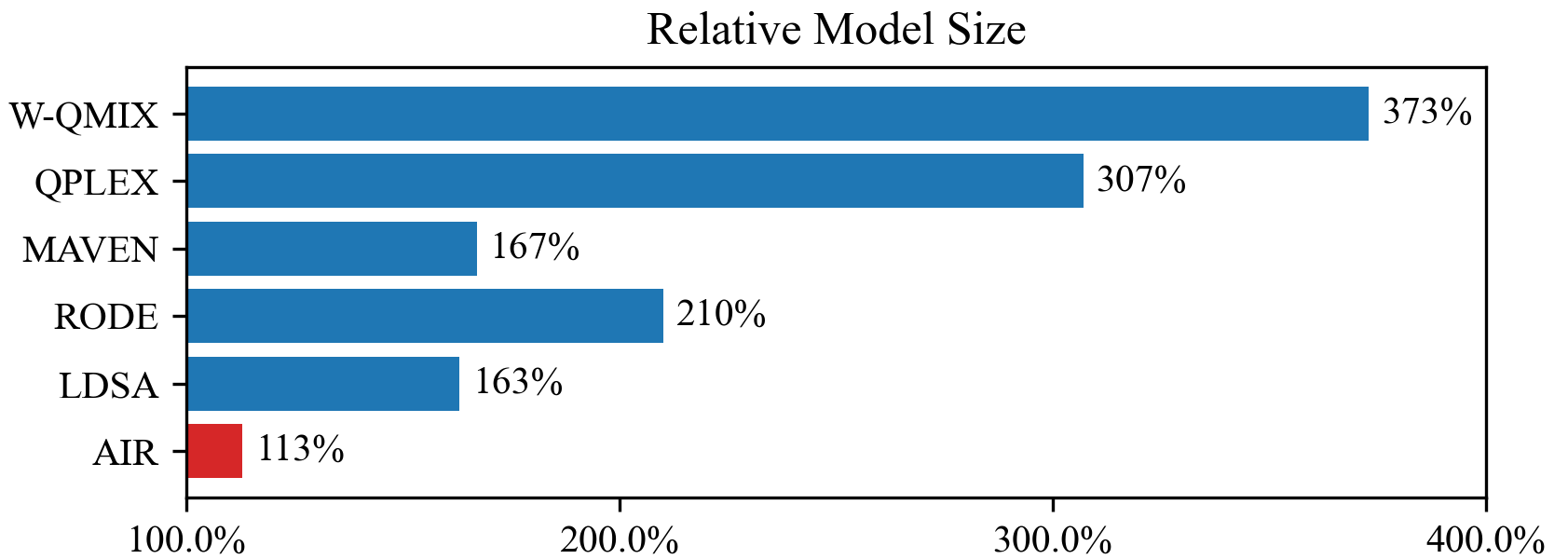}
    \caption{The relative model sizes of algorithms.}
    \label{fig:air_size}
\end{figure}

The primary advantage of AIR over other exploration methods for cooperative MARL lies in its integration of individual and collective exploration within a unified framework, as well as allowing for dynamic adjustment of exploration modes and intensity based on training progress. We first validate the benefits of unifying the two types of exploration and design two variants, AIR-in and AIR-co, which constrain the temperature value to be positive and negative during training, so the mode of exploration is set as \textbf{in}dividual and \textbf{co}llective respectively. We conduct experiments in SMAC \textit{MMM2} and \textit{6h\_vs\_8z} scenarios, and the results are shown in Figure~\ref{fig:air_mode}. When AIR incorporates only a single mode of exploration (AIR-in or AIR-co), its performance is substantially inferior to the original AIR, which integrates both individual and collective exploration modes. Therefore, the integration of individual and collective exploration within a unified framework is a key factor contributing to AIR's superior performance. 

\begin{figure}
    \centering
    \includegraphics[width=0.95\linewidth]{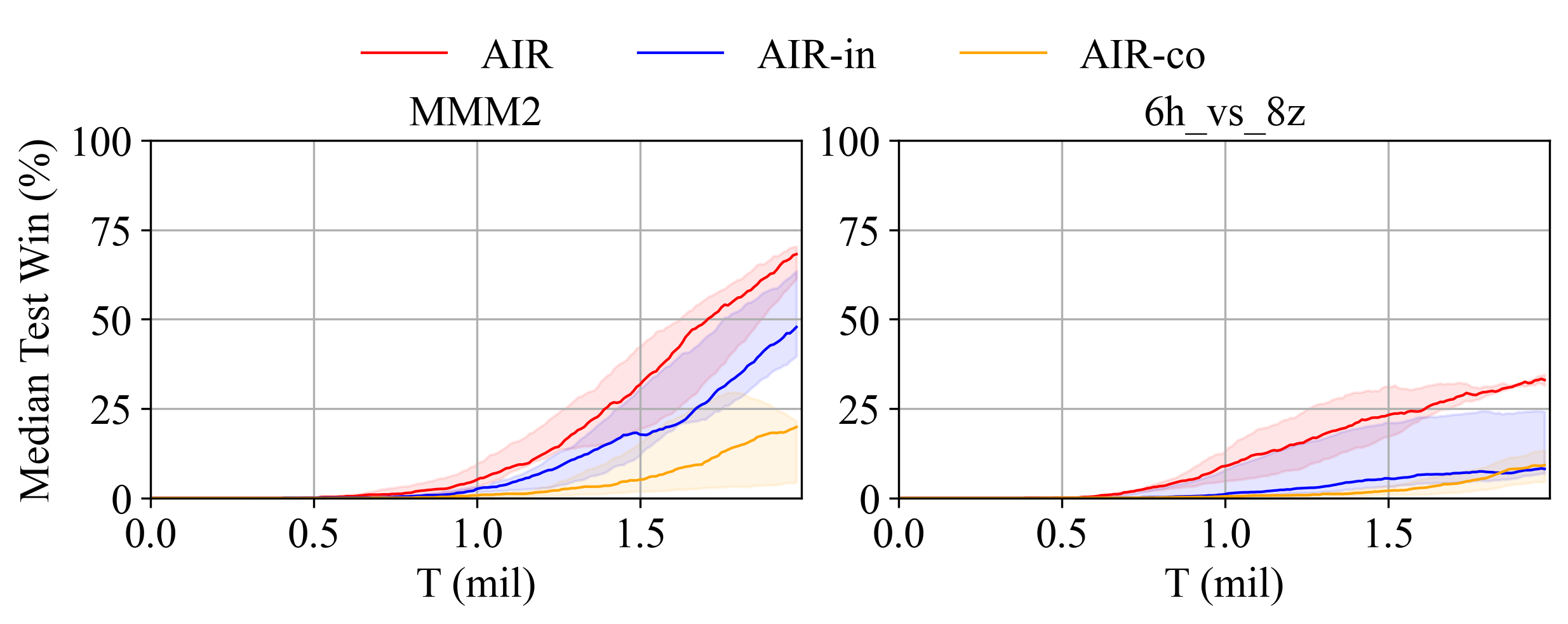}
    \caption{The experiment results of AIR and its variants with distinct exploration modes.}
    \label{fig:air_mode}
\end{figure}

Furthermore, we investigate the impact of dynamically adjusting the exploration degree (i.e., the temperature value) on the training process with a fixed exploration mode. However, selecting a proper temperature value $\alpha$ is non-trivial as mentioned above. Drawing from the study by \citet{hu2021rethinking}, we use QMIX with $\epsilon$-greedy exploration strategy as the control group and adjust the degree of individual exploration by switching the value of $\epsilon$ anneal period between $100K$ and $500K$. As shown in Figure~\ref{fig:air_degree}, with more intensive exploration, QMIX-500K performs better than QMIX-100K in \textit{6h\_vs\_8z}, but it would completely fail in \textit{corridor} as over-exploration ruins the exploitation of agents. Meanwhile, the degree of exploration in AIR can be dynamically adapted to the specific task and training process, rendering the success of AIR in both scenarios. Evidently, the mechanism of AIR for dynamically adjusting exploration intensity is of great superiority.

\begin{figure}
    \centering
    \includegraphics[width=0.95\linewidth]{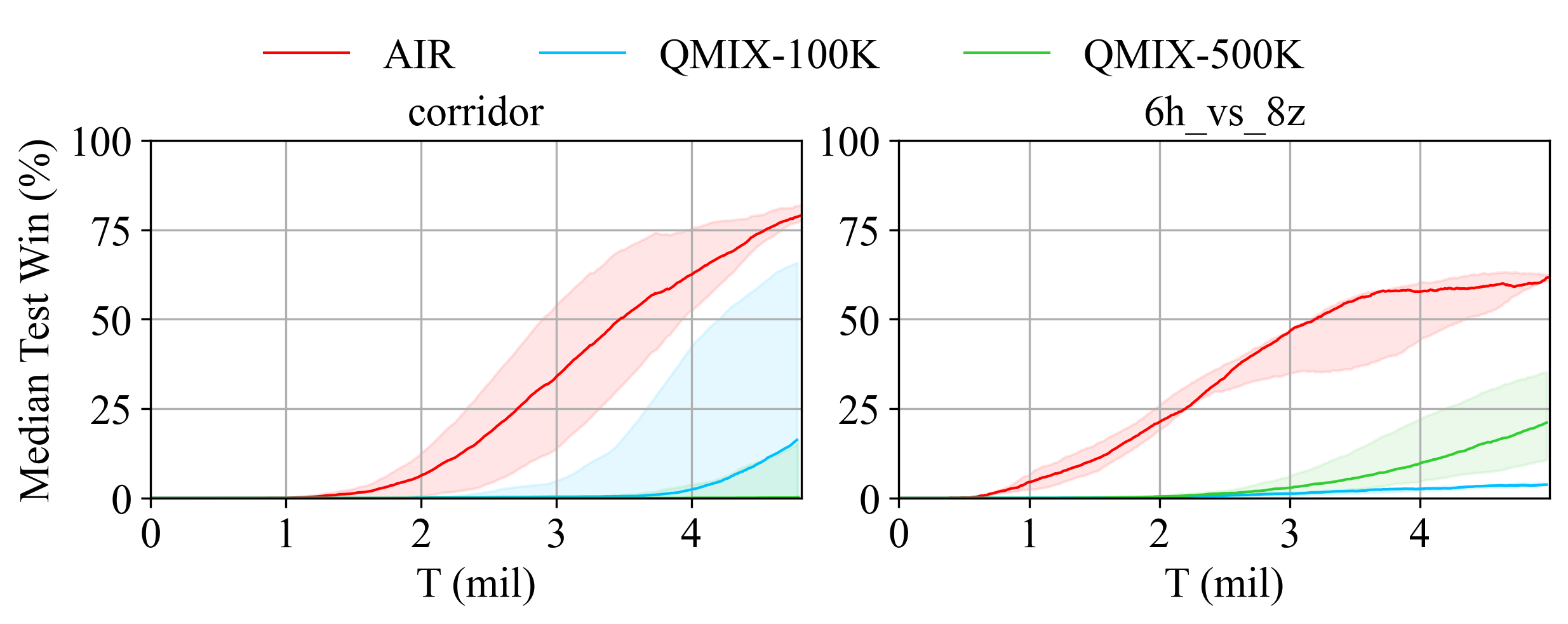}
    \caption{The experiment results of AIR and QMIX variants with varied degrees of individual exploration.}
    \label{fig:air_degree}
\end{figure}
\section{Conclusion}
Exploration for value-based agents in cooperative multi-agent tasks remains a tough problem, partly due to the absence of an explicit policy. Previous works have made some progress mainly through individual exploration based on agent's uncertainty towards the system or collective exploration focusing on agents' behavioral diversity. In this paper, we propose Adaptive exploration via Identity Recognition~(AIR), which is the very first approach to integrate both individual and collective in a unified framework. The core module in AIR is an identity classifier that distinguishes corresponding agents based on given trajectories. According to our theoretical analysis, the accuracy of the classifier is highly related to the exploration of agents. Consequently, we design a unified exploration mechanism that can switch the exploration mode between individual and collective exploration according to the training stage as well as tuning the degree of exploration. We conduct extensive experiments and studies on varied scenarios, and the results strongly demonstrate the superiority of our proposed method.

\bibliography{aaai25}

\begin{thebibliography}{31}
\providecommand{\natexlab}[1]{#1}

\bibitem[{Boyd and Vandenberghe(2004)}]{boyd2004convex}
Boyd, S.; and Vandenberghe, L. 2004.
\newblock \emph{Convex optimization}.
\newblock Cambridge university press.

\bibitem[{Goodfellow et~al.(2014)Goodfellow, Pouget-Abadie, Mirza, Xu, Warde-Farley, Ozair, Courville, and Bengio}]{goodfellow2014generative}
Goodfellow, I.; Pouget-Abadie, J.; Mirza, M.; Xu, B.; Warde-Farley, D.; Ozair, S.; Courville, A.; and Bengio, Y. 2014.
\newblock Generative adversarial nets.
\newblock \emph{Advances in neural information processing systems}, 27.

\bibitem[{Gu et~al.(2023)Gu, Kuba, Chen, Du, Yang, Knoll, and Yang}]{gu2023safe}
Gu, S.; Kuba, J.~G.; Chen, Y.; Du, Y.; Yang, L.; Knoll, A.; and Yang, Y. 2023.
\newblock Safe multi-agent reinforcement learning for multi-robot control.
\newblock \emph{Artificial Intelligence}, 319: 103905.

\bibitem[{Haarnoja et~al.(2017)Haarnoja, Tang, Abbeel, and Levine}]{haarnoja2017reinforcement}
Haarnoja, T.; Tang, H.; Abbeel, P.; and Levine, S. 2017.
\newblock Reinforcement learning with deep energy-based policies.
\newblock In \emph{International conference on machine learning}, 1352--1361. PMLR.

\bibitem[{Haarnoja et~al.(2018)Haarnoja, Zhou, Abbeel, and Levine}]{haarnoja2018soft}
Haarnoja, T.; Zhou, A.; Abbeel, P.; and Levine, S. 2018.
\newblock Soft actor-critic: Off-policy maximum entropy deep reinforcement learning with a stochastic actor.
\newblock In \emph{International conference on machine learning}, 1861--1870. PMLR.

\bibitem[{Hu et~al.(2021)Hu, Jiang, Harding, Wu, and Liao}]{hu2021rethinking}
Hu, J.; Jiang, S.; Harding, S.~A.; Wu, H.; and Liao, S.-w. 2021.
\newblock Rethinking the implementation tricks and monotonicity constraint in cooperative multi-agent reinforcement learning.
\newblock \emph{arXiv preprint arXiv:2102.03479}.

\bibitem[{Kurach et~al.(2020)Kurach, Raichuk, Stańczyk, Zając, Bachem, Espeholt, Riquelme, Vincent, Michalski, Bousquet, and Gelly}]{kurach2020google}
Kurach, K.; Raichuk, A.; Stańczyk, P.; Zając, M.; Bachem, O.; Espeholt, L.; Riquelme, C.; Vincent, D.; Michalski, M.; Bousquet, O.; and Gelly, S. 2020.
\newblock Google Research Football: A Novel Reinforcement Learning Environment.
\newblock arXiv:1907.11180.

\bibitem[{Li et~al.(2021)Li, Wang, Wu, Zhao, Yang, and Zhang}]{li2021celebrating}
Li, C.; Wang, T.; Wu, C.; Zhao, Q.; Yang, J.; and Zhang, C. 2021.
\newblock Celebrating diversity in shared multi-agent reinforcement learning.
\newblock \emph{Advances in Neural Information Processing Systems}, 34: 3991--4002.

\bibitem[{Mahajan et~al.(2019)Mahajan, Rashid, Samvelyan, and Whiteson}]{mahajan2019maven}
Mahajan, A.; Rashid, T.; Samvelyan, M.; and Whiteson, S. 2019.
\newblock Maven: Multi-agent variational exploration.
\newblock \emph{Advances in Neural Information Processing Systems}, 32.

\bibitem[{Mei, Zhou, and Lan(2024)}]{mei2024projection}
Mei, Y.; Zhou, H.; and Lan, T. 2024.
\newblock Projection-Optimal Monotonic Value Function Factorization in Multi-Agent Reinforcement Learning.
\newblock In \emph{AAMAS}, 2381--2383.

\bibitem[{Oliehoek and Amato(2016)}]{oliehoek2016concise}
Oliehoek, F.~A.; and Amato, C. 2016.
\newblock \emph{A concise introduction to decentralized POMDPs}.
\newblock Springer.

\bibitem[{Rashid et~al.(2020)Rashid, Farquhar, Peng, and Whiteson}]{rashid2020weighted}
Rashid, T.; Farquhar, G.; Peng, B.; and Whiteson, S. 2020.
\newblock Weighted qmix: Expanding monotonic value function factorisation for deep multi-agent reinforcement learning.
\newblock \emph{Advances in neural information processing systems}, 33: 10199--10210.

\bibitem[{Rashid et~al.(2018)Rashid, Samvelyan, Schroeder, Farquhar, Foerster, and Whiteson}]{rashid2018qmix}
Rashid, T.; Samvelyan, M.; Schroeder, C.; Farquhar, G.; Foerster, J.; and Whiteson, S. 2018.
\newblock Qmix: Monotonic value function factorisation for deep multi-agent reinforcement learning.
\newblock In \emph{International conference on machine learning}, 4295--4304. PMLR.

\bibitem[{Samvelyan et~al.(2019)Samvelyan, Rashid, De~Witt, Farquhar, Nardelli, Rudner, Hung, Torr, Foerster, and Whiteson}]{samvelyan2019starcraft}
Samvelyan, M.; Rashid, T.; De~Witt, C.~S.; Farquhar, G.; Nardelli, N.; Rudner, T.~G.; Hung, C.-M.; Torr, P.~H.; Foerster, J.; and Whiteson, S. 2019.
\newblock The starcraft multi-agent challenge.
\newblock \emph{arXiv preprint arXiv:1902.04043}.

\bibitem[{Shamsoshoara et~al.(2019)Shamsoshoara, Khaledi, Afghah, Razi, and Ashdown}]{shamsoshoara2019distributed}
Shamsoshoara, A.; Khaledi, M.; Afghah, F.; Razi, A.; and Ashdown, J. 2019.
\newblock Distributed cooperative spectrum sharing in uav networks using multi-agent reinforcement learning.
\newblock In \emph{2019 16th IEEE Annual Consumer Communications \& Networking Conference (CCNC)}, 1--6. IEEE.

\bibitem[{Son et~al.(2019)Son, Kim, Kang, Hostallero, and Yi}]{DBLP:journals/corr/abs-1905-05408}
Son, K.; Kim, D.; Kang, W.~J.; Hostallero, D.; and Yi, Y. 2019.
\newblock {QTRAN:} Learning to Factorize with Transformation for Cooperative Multi-Agent Reinforcement Learning.
\newblock \emph{CoRR}, abs/1905.05408.

\bibitem[{Sunehag et~al.(2017)Sunehag, Lever, Gruslys, Czarnecki, Zambaldi, Jaderberg, Lanctot, Sonnerat, Leibo, Tuyls et~al.}]{sunehag2017value}
Sunehag, P.; Lever, G.; Gruslys, A.; Czarnecki, W.~M.; Zambaldi, V.; Jaderberg, M.; Lanctot, M.; Sonnerat, N.; Leibo, J.~Z.; Tuyls, K.; et~al. 2017.
\newblock Value-decomposition networks for cooperative multi-agent learning.
\newblock \emph{arXiv preprint arXiv:1706.05296}.

\bibitem[{Tampuu et~al.(2017)Tampuu, Matiisen, Kodelja, Kuzovkin, Korjus, Aru, Aru, and Vicente}]{tampuu2017multiagent}
Tampuu, A.; Matiisen, T.; Kodelja, D.; Kuzovkin, I.; Korjus, K.; Aru, J.; Aru, J.; and Vicente, R. 2017.
\newblock Multiagent cooperation and competition with deep reinforcement learning.
\newblock \emph{PloS one}, 12(4): e0172395.

\bibitem[{Van~der Maaten and Hinton(2008)}]{van2008visualizing}
Van~der Maaten, L.; and Hinton, G. 2008.
\newblock Visualizing data using t-SNE.
\newblock \emph{Journal of machine learning research}, 9(11).

\bibitem[{Wang et~al.(2020{\natexlab{a}})Wang, Ren, Liu, Yu, and Zhang}]{wang2020qplex}
Wang, J.; Ren, Z.; Liu, T.; Yu, Y.; and Zhang, C. 2020{\natexlab{a}}.
\newblock Qplex: Duplex dueling multi-agent q-learning.
\newblock \emph{arXiv preprint arXiv:2008.01062}.

\bibitem[{Wang et~al.(2020{\natexlab{b}})Wang, Dong, Lesser, and Zhang}]{wang2020roma}
Wang, T.; Dong, H.; Lesser, V.; and Zhang, C. 2020{\natexlab{b}}.
\newblock Roma: Multi-agent reinforcement learning with emergent roles.
\newblock \emph{arXiv preprint arXiv:2003.08039}.

\bibitem[{Wang et~al.(2020{\natexlab{c}})Wang, Gupta, Mahajan, Peng, Whiteson, and Zhang}]{wang2020rode}
Wang, T.; Gupta, T.; Mahajan, A.; Peng, B.; Whiteson, S.; and Zhang, C. 2020{\natexlab{c}}.
\newblock Rode: Learning roles to decompose multi-agent tasks.
\newblock \emph{arXiv preprint arXiv:2010.01523}.

\bibitem[{Wang et~al.(2019)Wang, Wang, Wu, and Zhang}]{wang2019influence}
Wang, T.; Wang, J.; Wu, Y.; and Zhang, C. 2019.
\newblock Influence-based multi-agent exploration.
\newblock \emph{arXiv preprint arXiv:1910.05512}.

\bibitem[{Xu et~al.(2018)Xu, Lyu, Pan, Hu, Zhao, and Liu}]{xu2018multi}
Xu, Z.; Lyu, Y.; Pan, Q.; Hu, J.; Zhao, C.; and Liu, S. 2018.
\newblock Multi-vehicle flocking control with deep deterministic policy gradient method.
\newblock In \emph{2018 IEEE 14th International Conference on Control and Automation (ICCA)}, 306--311. IEEE.

\bibitem[{Yang et~al.(2022)Yang, Zhao, Hu, Zhou, Zhu, and Li}]{yang2022ldsa}
Yang, M.; Zhao, J.; Hu, X.; Zhou, W.; Zhu, J.; and Li, H. 2022.
\newblock Ldsa: Learning dynamic subtask assignment in cooperative multi-agent reinforcement learning.
\newblock \emph{Advances in Neural Information Processing Systems}, 35: 1698--1710.

\bibitem[{Yang et~al.(2020)Yang, Hao, Liao, Shao, Chen, Liu, and Tang}]{yang2020qatten}
Yang, Y.; Hao, J.; Liao, B.; Shao, K.; Chen, G.; Liu, W.; and Tang, H. 2020.
\newblock Qatten: A general framework for cooperative multiagent reinforcement learning.
\newblock \emph{arXiv preprint arXiv:2002.03939}.

\bibitem[{Zhang et~al.(2023{\natexlab{a}})Zhang, Cao, Yuan, Yu, and Zhan}]{zhang2023self}
Zhang, S.; Cao, J.; Yuan, L.; Yu, Y.; and Zhan, D.-C. 2023{\natexlab{a}}.
\newblock Self-motivated multi-agent exploration.
\newblock \emph{arXiv preprint arXiv:2301.02083}.

\bibitem[{Zhang et~al.(2023{\natexlab{b}})Zhang, Han, Wang, and Miao}]{zhang2023spatial}
Zhang, Z.; Han, S.; Wang, J.; and Miao, F. 2023{\natexlab{b}}.
\newblock Spatial-temporal-aware safe multi-agent reinforcement learning of connected autonomous vehicles in challenging scenarios.
\newblock In \emph{2023 IEEE International Conference on Robotics and Automation (ICRA)}, 5574--5580. IEEE.

\bibitem[{Zhao et~al.(2023)Zhao, Lin, Chapman, Metje, and Hao}]{zhao2023optimizing}
Zhao, G.; Lin, K.; Chapman, D.; Metje, N.; and Hao, T. 2023.
\newblock Optimizing energy efficiency of LoRaWAN-based wireless underground sensor networks: A multi-agent reinforcement learning approach.
\newblock \emph{Internet of Things}, 22: 100776.

\bibitem[{Zheng et~al.(2021)Zheng, Chen, Wang, He, Hu, Chen, Fan, Gao, and Zhang}]{zheng2021episodic}
Zheng, L.; Chen, J.; Wang, J.; He, J.; Hu, Y.; Chen, Y.; Fan, C.; Gao, Y.; and Zhang, C. 2021.
\newblock Episodic multi-agent reinforcement learning with curiosity-driven exploration.
\newblock \emph{Advances in Neural Information Processing Systems}, 34: 3757--3769.

\bibitem[{Zhou et~al.(2023)Zhou, Xu, Zhang, and Fan}]{zhou2023sora}
Zhou, G.; Xu, Z.; Zhang, Z.; and Fan, G. 2023.
\newblock SORA: Improving Multi-agent Cooperation with a Soft Role Assignment Mechanism.
\newblock In \emph{International Conference on Neural Information Processing}, 319--331. Springer.

\end{thebibliography}

\newpage
\onecolumn
\appendix
\setcounter{secnumdepth}{2}
\setcounter{theorem}{0}
\section{Theoretical Derivations}   \label{sec:proof}
\subsection{The Decomposition of KL-Divergence}\label{sec:lemma1}
\begin{lemma}
    Given the system trajectory visit distribution $\rho$, of which the entropy $\mathcal{H}(\rho)$ can be decomposed as below:
    \begin{equation*}
        \mathcal{H} = \mathbb{E}_{z}\left[\mathcal{D}_{\textrm{KL}}\left(\rho(\tau, u|z) || \rho(\tau, u)\right)\right] + \mathcal{H}(\rho|z)
    \end{equation*}
\end{lemma}

\begin{proof}
    \begin{align*}
        \mathcal{H}(\rho) &= \mathbb{E}_{(\tau, u)\sim \rho}\left[-\log \rho(\tau, u) \right] \\
        &= \mathbb{E}_{(\tau, u, z)\sim \rho} \left[\log \frac{\rho(\tau, u|z)}{\rho(\tau,u)} - \log \rho(\tau, u | z) \right] \\
        &= \mathbb{E}_{z}\left[\mathcal{D}_{\textrm{KL}}\left(\rho(\tau, u|z) || \rho(\tau, u)\right)\right] + \mathcal{H}(\rho|z)
    \end{align*}
\end{proof}

\subsection{Equivalent Optimization Target}
\begin{lemma}
    \begin{equation*}
        \mathcal{H}(z| \rho) \propto -\mathcal{D}_{\textrm{KL}}\left[\rho(\tau, u|z)||\frac{1}{n}\sum\limits_{k=1}^n \rho(s, a|z_k)\right]
    \end{equation*}
\end{lemma}

\begin{proof}
    According to the mutual information theory, the MI between the trajectory visit distribution and the agent identity can be calculated as:
    \begin{equation*}
        \mathcal{I}(\rho ; z) = \mathcal{H}(\rho)-\mathcal{H}(\rho|z) = \mathcal{H}(z) - \mathcal{H}(z|\rho)
    \end{equation*}
    
    $\mathcal{H}(z)$ is a constant as mentioned in the paper. Thus we have:
    \begin{align*}
        \mathcal{H}(z|\rho) &= \mathcal{H}(z) + \mathcal{H}(\rho|z) - \mathcal{H}(\rho) \\
        &\propto \mathbb{E}_{(\tau, u, z)\sim \rho}\left[-\log \rho(\tau, u|z)\right] - \mathbb{E}_{(\tau, u)~\rho}\left[-\log \rho(\tau, u)\right] \\
        &= \mathbb{E}_{(\tau,u,z)\sim\rho}\left[-\log(\rho(\tau,u|z))\right]-\int-\rho(\tau,u)\log(\rho(\tau,u))d\tau du \\
        &= \mathbb{E}_{(\tau,u,z)\sim\rho}\left[-\log(\rho(\tau,u|z))\right]-\int-\rho(\tau,u,z)\log(\rho(\tau,u))d\tau du dz \\
        &= \mathbb{E}_{(\tau,u,z)\sim\rho}\left[\log (\rho(\tau,u)-\log (\rho(\tau,u|z)\right]
    \end{align*}
    
    Using total probability theorem, 
    \begin{equation*}
        \rho(\tau, u)=\int \rho(\tau, u|z)p(z)dz=\frac{1}{n}\sum\limits_{k=1}^N \rho(\tau, u|z_k)
    \end{equation*}
    
    Substituting it into the formula above, we obtain:
    \begin{align*}
        \mathcal{H}(z|\rho) &\propto \mathbb{E}_{(\tau,u,z)\sim\rho}\left[\frac{1}{n}\sum\limits_{k=1}^n \rho(\tau, u|z_k)-\log (\rho(\tau,u|z)\right] \\
        &\propto -\mathcal{D}_{\textrm{KL}}\left[\rho(\tau, u|z)||\frac{1}{n}\sum\limits_{k=1}^N \rho(s, a|z_k)\right]
    \end{align*}
    
\end{proof}
\subsection{The Relationship between Identity Classifier and Policy}\label{sec:lemma3}
\begin{lemma}
The identity classifier $q_{\zeta}(z_k|\tau, u)$ can be used to measure action selection probabilities of each agent, as the following relationship holds:
    \begin{equation*}
    p(z_k|\tau_T, u_T)=\frac{\prod\limits_{t=0}^T \pi^k(u_t|o_t)}{\prod\limits_{t=0}^T \pi^k(u_t|o_t) + \sum\limits_{\substack{i=1\\ i\neq k}}^n \prod\limits_{t=0}^T \pi^i(u_t|o_t)}
\end{equation*}
\end{lemma}

\begin{proof}
    In a mini-batch $\mathcal{B}$, since the number of samples belonging to each agent is the same, we have $p(z_k)=\frac{1}{n}$ for any agent $k$. Combining with the definition of trajectory visit distribution in Equation~\ref{eq:rho_k} and \ref{eq:rho}, we can transform the posterior $p(z_k|\tau_T, u_T)$ as below:

    \begin{align*}
        p(z_k|\tau_T, u_T)&=\frac{p(z_k, \tau_T, u_T)}{p(\tau_T, u_T)} \\
        &=\frac{p(\tau_T, u_T|z_k)\cdot p(z_k)}{\sum\limits_{i=1}^n p(\tau_T, u_T|z_i)\cdot p(z_i)}\\
        &=\frac{1}{n}\cdot \frac{\rho^k(\tau_T, u_T)}{\rho(\tau_T, u_T)}\\
        &=\frac{\prod\limits_{t=0}^T \pi^k(u_t|o_t)\sum\limits_{s_t}P(s_t)\mathcal{O}(o_t|s_t)}{\sum\limits_{i=1}^n \prod\limits_{t=0}^T \pi^i(u_t|o_t)\sum\limits_{s_t}P(s_t)\mathcal{O}(o_t|s_t)}\\
        &=\frac{\prod\limits_{t=0}^T \sum\limits_{s_t}P(s_t)\mathcal{O}(o_t|s_t) \cdot \prod\limits_{t=0}^T \pi^k(u_t|o_t)}{\prod\limits_{t=0}^T \sum\limits_{s_t}P(s_t)\cdot \sum\limits_{i=1}^n \prod\limits_{t=0}^T \pi^k(u_t|o_t)}\\
        &=\frac{\prod\limits_{t=0}^T \pi^k(u_t|o_t)}{\sum\limits_{i=1}^n \prod\limits_{t=0}^T \pi^k(u_t|o_t)}
    \end{align*}
    Separate the term related to agent $k$ in the denominator, we get:
    \begin{equation*}
    p(z_k|\tau_T, u_T)=\frac{\prod\limits_{t=0}^T \pi^k(u_t|o_t)}{\prod\limits_{t=0}^T \pi^k(u_t|o_t) + \sum\limits_{\substack{i=1\\ i\neq k}}^n \prod\limits_{t=0}^T \pi^i(u_t|o_t)}
    \end{equation*}
    As the policies of each agent are relatively independent within a training step, the value of the second term in the denominator is not affected by $\pi^k$. Therefore, there is a positive correlation between the probability of an action in the trajectory being selected by agent $k$ and the likelihood of the trajectory being identified as belonging to agent $k$. As the classifier learns a posterior $q_{\zeta}$ to approximate $p(z_k|\tau_T, u_T)$, we can leverage it to assess action selection probabilities without an explicit policy function.
\end{proof}

\section{Environment Details}
\subsection{SMAC}
SMAC is a simulation environment for research in collaborative multi-agent reinforcement learning (MARL) based on Blizzard's StarCraft II RTS game. It provides various micro-battle scenarios and also supports customized scenarios for users to test the algorithms. The goal in each scenario is to control different types of ally agents to move or attack to defeat the enemies. The enemies are controlled by a heuristic built-in AI with adjustable difficulty level between 1 to 7. In our experiments, the difficulty of the game AI is set to the highest~(the 7th level). The version of StarCraft II is 4.6.2 (B69232) in our experiments, and it should be noted that results from different client versions are not always comparable. Table~\ref{tab:challenges} presents the details of selected scenarios in our experiments.

\begin{table}[h]
\centering
\caption{Information of selected challenges.}
\resizebox{0.75\linewidth}{!}{
\begin{tabular}{lcccc}
\hline
Challenge&Ally Units&Enemy Units&Type&Level of Difficulty\\
\hline
1c3s5z&\makecell[c]{1 Colossi\\3 Stalkers\\5 Zealots}&\makecell[c]{1 Colossi\\3 Stalkers\\5 Zealots}&\makecell[c]{Heterogeneous\\Symmetric}&Hard\\
\hline
5m\_vs\_6m&5 Marines&6 Marines&\makecell[c]{Homogeneous\\Asymmetric}&Hard\\
\hline
3s5z\_vs\_3s6z&\makecell[c]{3 Stalkers\\5 Zealots}&\makecell[c]{3 Stalkers\\6 Zealots}&\makecell[c]{Heterogeneous\\Asymmetric}&Super Hard\\
\hline 
MMM2&\makecell[c]{1 Medivac\\2 Marauders\\7 Marines}&\makecell[c]{1 Medivac\\3 Marauders\\8 Marines}&\makecell[c]{Heterogeneous\\Asymmetric}&Super Hard\\
\hline
corridor&6 Zealots&24 Zerglings&\makecell[c]{Homogeneous\\Asymmetric}&Super Hard\\
\hline
27m\_vs\_30m&27 Marines&30 Marines&\makecell[c]{Homogeneous\\Asymmetric}&Super Hard\\
\hline
\end{tabular}}
\label{tab:challenges}
\end{table}

\subsection{SMACv2} \label{sec:smacv2}
SMACv2 (Ellis et al. 2023) is proposed to address SMAC's lack of stochasticity. In SMACv2, the attack ranges of different unit types are no longer the same. Besides, the field of view of the agents is further restricted to a sector shape rather than a circle. Therefore, the agent receives less information and suffers more severe partial observability problems. The team compositions and agent start positions are generated randomly at the beginning of each episode. These modifications make SMACv2 extremely challenging. It is worth noting that the disparity between the lineups of the two sides can be substantial in some episodes due to the randomness in initialization, so it is impossible for an algorithm to reach a 100\% win rate. The version of the StarCraft II engine is also 4.6.2 (B69232) in our experiments.

\subsection{Google Research Football}
Google Research Football~(GRF)~\citep{kurach2020google} is a simulation football environment developed by Google Research and has been well received by the reinforcement learning community. It provides a user-friendly platform for training and evaluating agents in various areas, such as control, planning, and multi-agent cooperation. We choose three official scenarios from Football Academy. \textit{Academy\_pass\_and\_shoot\_with\_keeper} is a relatively easy scenario, while \textit{academy\_corner} and \textit{academy\_run\_pass\_and\_shoot} require more sophisticated coordination and are much harder. Agents are rewarded when they score a goal or kick the ball to a position close to the goal. Observations of the agent include the relative positions of all other entities. We restrict the football to the opponent's half of the field and stop the episode once the ball enters our half to speed up training.

\section{Implementation Details}
\subsection{Settings of Hyperparameters}    \label{sec:hyper}
We list the hyperparameters of the AIR in Table~\ref{tab:hypers}. The hyperparameters of the baselines in our experiments remain the same as their official implementations.

\begin{table}[h]
\centering
\caption{The hyperparameter settings.}
\begin{tabular}{ll}
\hline
Description&Value\\
\hline
Type of value mixer&QMIX\\
Dimension of hidden states in $q_{\zeta}$&64\\
Dimension of hidden states in RNN&64\\
Dimension of the mixing network&32\\
Dimension of hypernetworks&64\\
Batch size&32\\
Trajectories sampled per run&1\\
Replay buffer size&5000\\
Discount factor $\gamma$&0.99\\
Probability of random action ($\epsilon$)&1.0$\sim$0.05\\
Anneal time for $\epsilon$&50000\\
Type of optimizer&Adam\\
Learning rate for $q_{\zeta}$&0.0005\\
Learning rate for $\alpha$&0.0005\\
Learning rate for agents&0.0005\\
Target network update interval&200\\
\hline
\end{tabular}
\label{tab:hypers}
\end{table}
\subsection{Experiments Compute Resources}  \label{sec:hardware}
We conducted our experiments on a platform with 2 Intel(R) Xeon(R) Platinum 8280 CPU \@ 2.70GHz processors, each with 26 cores. Besides, we use a GeForce RTX 3090 GPU to facilitate the training procedure. The time of execution varies by scenario.

\section{Additional Experiments}    \label{sec:sup_exp}
Due to the space limitations, we include the experiments on SMACv2 here. The details of SMACv2 environment can be found in Appendix~\ref{sec:smacv2}. At the beginning of an episode, the unit type and location of each agent are initialized randomly. Besides, the field of view of the agents is further restricted to a sector shape rather than a circle. Therefore, the agent receives less information and suffers more severe partial observability problems, which increases the demand for additional sources of information. It is worth noting that, due to the random initialization, the power gap between two sides could be too wide for the ally agents to win this episode. Therefore, there exists an upper bound for the win rates.

Since RODE~\citep{wang2020rode} and LDSA~\citep{yang2022ldsa} are two methods about multi-agent exploration and perform relatively well in SMAC (RODE even outperforms AIR in \textit{corridor}), we select them as baselines for further comparison. RODE and LDSA employ the episode runner to collect data, while AIR uses the parallel runner in our code to accelerate the training process but reduce the frequency of network updates. We expect a further improvement in AIR's performance upon switching to the episode runner. The experiment results are displayed in Figure~\ref{fig:scv2_results}. All the tested methods adopt the mixing network of QMIX for fairness. We observe that RODE~\citep{wang2020rode} suffers a severe deterioration in the effectiveness of policies, which indicates that it may learn open-loop policies in SMAC and could not deal with the randomness in SMACv2. The performance of LDSA~\citep{yang2022ldsa} is relatively better but still far from the highest win rate records. On the contrary, AIR is significantly superior and even exceeds the highest win rate records of all baselines in the original paper of SMACv2 (Ellis et al. 2023).

\begin{figure*}[h]
    \centering
    \includegraphics[width=0.95\textwidth]{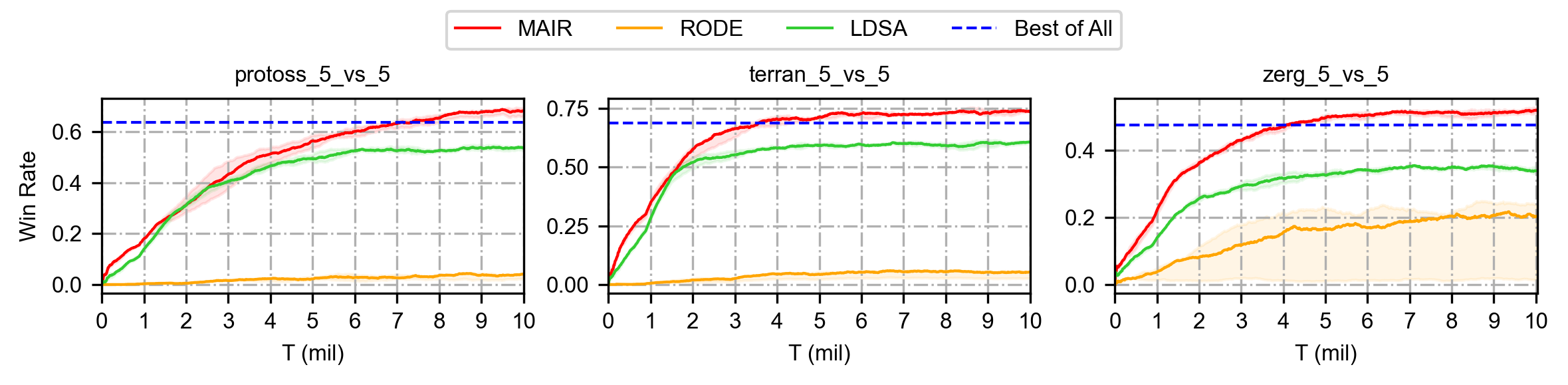}
    \caption{Comparison of AIR against baselines and the highest win rate records (in blue dotted line) in SMACv2 scenarios.}
    \label{fig:scv2_results}
\end{figure*}

\end{document}